\documentclass[twocolumn]{article} 
\usepackage{times}  
\usepackage{helvet}  
\usepackage{courier}  
\usepackage[hyphens]{url}  
\usepackage{graphicx} 
\usepackage{natbib}  
\usepackage{caption} 
\usepackage{natbib}
\usepackage{algorithm}
\usepackage{algorithmic}
\usepackage{thm-restate}
\usepackage{newfloat}
\usepackage{listings}
\DeclareCaptionStyle{ruled}{labelfont=normalfont,labelsep=colon,strut=off} 
\floatstyle{ruled}
\newfloat{listing}{tb}{lst}{}
\floatname{listing}{Listing}
\usepackage{xspace}
\usepackage{mathtools}
\usepackage{amsmath,amsfonts}
\usepackage{amssymb}
\usepackage{amsthm}
\usepackage[capitalize]{cleveref}
\usepackage{fdsymbol}
\usepackage{multirow}
\usepackage{subfig}
\usepackage{todonotes}
\usepackage{siunitx}
\usepackage{booktabs}

\newcommand{\tabularexample}{\ensuremath{x}\xspace}
\newcommand{\setofexamples}{\ensuremath{\Smc}\xspace}
\theoremstyle{definition}

\newcommand{\hsp}{\mathcal{H}}
\DeclareMathOperator{\efn}{\ensuremath{\mathsf{error}}\xspace}
\newcommand{\samples}{\mathcal{S}}
\newcommand{\trepan}{\textsc{TrePAC}\xspace}

\newtheorem{theorem}{Theorem}

\newtheorem{definition}[theorem]{Definition}
\newtheorem{corollary}[theorem]{Corollary}
\newtheorem{lemma}[theorem]{Lemma}

\input{misc0.sty}

\newcommand{\feature}{\ensuremath{f}\xspace} 
\newcommand{\bestsplit}{\ensuremath{\mathsf{best\_{split}}}\xspace}

\newcommand{\e}{\ensuremath{e}\xspace} 
\newcommand{\setunclassexamples}{\ensuremath{\Smc_{\mathbf{e}}}\xspace}

\newcommand{\tree}{\ensuremath{T}\xspace}

\newcommand{\Dmc}{D\xspace}
\newcommand{\examples}{\ensuremath{\Emc}\xspace}
\newcommand{\hypothesisSpace}{\ensuremath{\Hmc}\xspace}
\newcommand{\hypothesisSpaceDT}{\ensuremath{\Hmc_{DT}}\xspace}
\newcommand{\TopDown}{\text{{TopDown}}\xspace}
\newcommand{\Trepan}{\textsc{TrePAC}}
\newcommand{\trepac}{\textsc{TrePAC}\xspace}
\newcommand{\probabilitydistribution}{\ensuremath{\mathcal{D}}\xspace}
\newcommand{\EX}[3]{\ensuremath{{\sf EX}^{#1}_{#2,#3}}\xspace}
\newcommand{\bbC}{\ensuremath{\mathbb{C}}\xspace}
\newcommand{\bbCDT}{\ensuremath{\mathbb{C}}_{DT}}
\newcommand{\setofconstraints}{\ensuremath{\Phi}\xspace}

\newcommand{\arxivurl}{\citet{ozaki2024extractingpacdecisiontrees}} 

\newif\IFproofread

\title{Extracting PAC Decision Trees from Black Box Binary Classifiers:   The Gender Bias Case Study on BERT-based Language Models}
\author{
	Ana Ozaki\textsuperscript{{\rm 1},{\rm 3}}, Roberto Confalonieri\textsuperscript{\rm 2}, Ricardo Guimarães\textsuperscript{\rm 3}, Anders Imenes\textsuperscript{\rm 3}
	\\
	\textsuperscript{\rm 1}Universitetet i Oslo, Norway\\
	\textsuperscript{\rm 2}Universita degli Studi di Padova, Italy\\
	\textsuperscript{\rm 3}Universitetet i Bergen, Norway\\
	anaoz@uio.no, roberto.confalonieri@unipd.it
}
\date{}

\begin{document}

\maketitle

\begin{abstract}
	Decision trees are a popular machine learning method, valued for their inherent explainability. In Explainable AI, decision trees serve as surrogate models for complex black box AI models or as approximations of parts of such models. 
	A key challenge of this approach is assessing how accurately the extracted decision tree represents the original model and determining the extent to which it can be trusted as an approximation of its behavior. 
	In this work, we investigate the use of the Probably Approximately Correct (PAC) framework to provide a theoretical guarantee of fidelity for decision trees extracted from AI models. Leveraging the theoretical foundations of the PAC framework, we adapt a decision tree algorithm to ensure a PAC guarantee under specific conditions. We focus on binary classification and conduct experiments where we extract decision trees from BERT-based language models with PAC guarantees. Our results indicate occupational gender bias in these models, which confirm previous results in the literature. Additionally, the decision tree format enhances the visualization of which occupations are most impacted by social bias.
\end{abstract}

%

\section{Introduction}
Decision trees are known to be easy to explain because, given an input, one can directly trace the decisions which led to the output. 
However, this method lacks the generalizability of architectures based on neural networks that have hidden layers.  Although there has been some progress in increasing the generalizability of decision trees~\cite{DBLP:conf/aiia/FrosstH17} and even translating trained neural networks into decision trees~\cite{DBLP:journals/corr/abs-2210-05189}, there appears to be a trade-off between generalizability and interpretability that is not easy to balance.

To enjoy the best of both worlds,  the AI research community has been exploring the idea of \emph{extracting decision trees from trained neural networks}. The extracted trees can then serve as an explainable approximation of these black-box models~\cite{ECAI-2020,DBLP:journals/pr/KrishnanSB99,DBLP:conf/smartgreens/VasilevMN20,DBLP:conf/kdd/Boz02,DBLP:journals/acisc/BolognaH18,DBLP:journals/tnn/SchmitzAG99,DBLP:conf/flairs/DanceyMB04,DBLP:journals/frai/BurkhardtBWAKK21,DBLP:conf/ijcai/SetionoL95,DBLP:conf/uai/NanfackTF21,Craven1995,DBLP:journals/ijdats/YoungWRWPSM11,Awudu2015XTREPANAM} (see also \cite{DBLP:journals/tnn/ChorowskiZ11,DBLP:conf/sat/ShihDC19} for works extracting decision diagrams from trained neural networks). Previous works in this direction indicate that, in practice, decision trees are indeed useful for understanding patterns in the data, which cannot be understood from black-box models. 
This approach has been applied to various domains such as medicine~\cite{DBLP:conf/ijcnn/DanceyBM10}, traffic signal control~\cite{DBLP:journals/tcss/ZhuYC23}, risk credit evaluation~\cite{DBLP:journals/mansci/BaesensSMV03},
water resource management~\cite{Wei2012}, multimedia~\cite{DBLP:conf/itat/FantaPH19}, among others. 
The decision tree in this case functions as a \emph{surrogate} for the original model. 
The advantage of surrogate models is their ability to abstract the original model, focusing on specific inputs or aspects of interest.  

In this paper, we apply this approach to unveil occupational gender bias in BERT-based language models.  
By extracting rules that approximate \emph{parts of} the model of interest, we can avoid the challenges associated with approximating the whole language model, which could be computationally infeasible or inefficient. Moreover, the approach of extracting decision trees can be applied to black-box models regardless of their internal architectures, which are constantly evolving as the field of artificial intelligence tackles increasingly complex tasks. 

The main challenge with such approximations lies in determining \emph{how accurately the extracted decision tree represents the relevant parts of the original black box model}. 
To the best of our knowledge, current practical approaches for extracting decision trees from trained neural networks do not provide theoretical guarantees regarding the fidelity of the decision trees with respect to the original model. What guarantees would be appropriate? How could such guarantees be given? 
%
In Computational Learning Theory, the most well studied framework which gives correctness guarantees is known as the Probably Approximately Correct (PAC) framework~\cite{DBLP:conf/stoc/Valiant84}.  
In the original framework, a learner receives examples classified according to some \emph{target} function. The classification is binary and the classified examples come according to a probability distribution that is fixed but arbitrary and unknown. The learner then creates a \emph{hypothesis} consistent with the information received. In a nutshell, the theoretical correctness guarantee is that, provided with enough data, with high probability, the difference between any hypothesis created by the learner that is consistent with the data and the target function is minimal. 
This ``difference''  is quantified by the probability of misclassification, known as the \emph{true error}~\cite{UnderML2014}.
The \emph{sample complexity} quantifies the amount of data needed for such a guarantee ~\cite{doi:10.1137/1116025}.

In this work, we study the PAC framework and explore its potential to provide a theoretical correctness guarantee for decision trees extracted from black-box models. Our contributions can be summarized as follows:
\begin{itemize}
	\item We provide formal definitions for the relevant notions we use. In particular, we relate the notion of fidelity in the literature on decision trees with the PAC guarantee (\cref{sec:pac}). We focus on binary classification, as this is the most studied setting within the PAC framework.
	\item On the theoretical side, we provide a bound on the sample size needed for PAC learning when the training error is bounded by
	$k/|\Smc|$, where $|\Smc|$ is the size of the training set and $k\in\mathbb{N}$ (\cref{change_me}). This bound is used in a decision tree algorithm (\cref{sec:decisiontrees}). 
	\item On the practical side (\cref{sec:experiments}), we perform experiments based on a recent gender bias case study~\cite{BLUM2023109026}, where we extract decision trees using BERT-based models~\cite{DBLP:conf/naacl/DevlinCLT19,DBLP:journals/corr/abs-1907-11692}. 
\end{itemize}
We conclude and mention future work in \cref{sec:conclusion}.




\section{The PAC Framework}\label{sec:pac}

Here we   introduce the relevant notation for decision trees and PAC learning 
(\cref{subsec:basicpac}). 
We then show in \cref{subsec:trainingerror} 
our theoretical bound on the number of examples when the hypothesis is inconsistent with the training set.

\subsection{Basic Definitions}\label{subsec:basicpac}
\newcommand{\target}{\ensuremath{t}\xspace}
\newcommand{\leaves}[1]{\ensuremath{{\sf leaves}({#1})}\xspace}

We now formalise the PAC learning framework. We follow the notation for the definition of PAC learning 
adopted in the literature~\cite{DBLP:conf/stoc/Valiant84,UnderML2014,Kearns1999,DBLP:journals/ki/Ozaki20}.
\begin{definition}
	A \emph{concept class} $\mathbb{C}$ is a triple 
	$(\examples, \hypothesisSpace, \mu)$ where  
	$\hypothesisSpace$ is a set of concept representations\footnote{In the Machine Learning Theory literature,
		a \emph{concept} is often defined as a set of examples and a concept representation is a way of representing such set.}, called \emph{hypothesis space};
	$\examples$ is a set of examples; 
	and 
	$\mu$ is a function that maps each element of $\hypothesisSpace$ to a 
	subset of    $\examples$. 
\end{definition}
Each element of $\hypothesisSpace$ is called a \emph{hypothesis}.
The \emph{target representation} (here simply called \emph{target}) 
is a fixed but arbitrary element of $\hypothesisSpace$,  representing 
the kind of knowledge  that is aimed to be learned. 
Given a target $\target\in\hypothesisSpace$,
we say that the examples in $\mu(\target)$ are \emph{positive examples}, otherwise, they are called 
\emph{negative examples}.
A \emph{classified example} w.r.t. the target \target is
a pair $(\e,{\sf c}(\e))$, where ${\sf c}(\e)=1$ (positive) if $\e\in\mu(\target)$ and  ${\sf c}(\e)=0$ (negative) otherwise.
A \emph{training set} is a set of classified examples (we may omit   that this is w.r.t.   \target). Given a training set \setofexamples, we denote by \(\samples_e\)  the examples in \(\samples\).  In other words,  \(\samples_e\) is the result of removing the labels in  \(\samples\).  

\begin{definition}[Induced Probability Distribution] Let $m$ be a positive integer and let $\probabilitydistribution$ be a probability distribution over a set of examples \examples. 
	The probability distribution $\probabilitydistribution$ \emph{induces} the probability distribution $\probabilitydistribution^m$ over the set of all 
	sequences
	$\setunclassexamples$ of $\examples$ with $m$ elements, defined as
	\[\probabilitydistribution^m(\{\setunclassexamples \}):=\prod_{\e\in \setunclassexamples}\probabilitydistribution(\{\e\}).\]    
\end{definition}
%
%
%


\begin{definition}[Example Query]
	Let $\bbC=(\examples, \hypothesisSpace, \mu)$ be a concept class. 
	Given a target $\target \in \hypothesisSpace$, let 
	\EX{\probabilitydistribution}{\ensuremath{\mathbb{C}}}{\target} 
	be the {\em oracle} that takes no input, and outputs a
	classified example $(\e, {\sf c}(\e))$, where $\e \in \examples$ is drawn according to the probability
	distribution $\probabilitydistribution$ and ${\sf c}(\e) = 1$, if $\e \in \mu(\target)$, and ${\sf c}(\e) = 0$, otherwise. An \emph{example
		query} is a call to the oracle \EX{\probabilitydistribution}{\ensuremath{\mathbb{C}}}{\target}. A sample generated by 
	\EX{\probabilitydistribution}{\ensuremath{\mathbb{C}}}{\target}  is a 
	multiset of
	classified examples, independently and identically distributed according
	to $\probabilitydistribution$, drawn by calling \EX{\probabilitydistribution}{\ensuremath{\mathbb{C}}}{\target}.
\end{definition}

\begin{definition}[PAC learnability]
	A concept class \(\bbC\) is {\em PAC-learnable} if there is a function $f:(0,1/2)^2 \rightarrow \mathbb{N}$
	and a deterministic algorithm s.t.\ for every $\epsilon, \delta  \in (0,1/2)   \subseteq \mathbb{R}$, every probability distribution $\probabilitydistribution$ on $\examples$ and every target $\target \in \hypothesisSpace$, given a sample of size $m \geq f(\epsilon,\delta)$ generated by $\EX{\probabilitydistribution}{\mathbb{C}}{\target}$, the algorithm always halts and outputs $h \in \hypothesisSpace$ such that with probability at least $1-\delta$ over the choice of $m$ examples (counting repetitions) in $\examples$, we have  $\probabilitydistribution(\mu(h) \oplus \mu(\target)) \leq \epsilon$.
\end{definition}
The above condition can also be written 
as  $\probabilitydistribution^{m} (\{\setunclassexamples \mid \probabilitydistribution(\mu(\target)  \oplus \mu(h)) \leq \epsilon\}) \geq (1-\delta)$. 
Given a training set $\setofexamples$, let $h_{\setofexamples}$ be a fixed but arbitrary element of \(\hypothesisSpace\) such that, for all 
$(\e,{\sf c}(\e))\in \setofexamples$, ${\sf c}(\e)=1$ iff $e\in\mu(h_\setofexamples)$ (i.e., $h_\setofexamples$ and $\target$ agree on their classification of the elements of $\setunclassexamples$). 
Since $\target\in \hypothesisSpace$ such $h_\setofexamples$   exists. 
The error of $h$ is the probability that $h$ misclassifies the examples drawn according to $\probabilitydistribution$.

\begin{definition}[True Error]\label{def:trueerror}
	Given a hypothesis  $h\in\hypothesisSpace$, a target function $\target$ and a probability distribution $\probabilitydistribution$. The true error of $h$ w.r.t. $\target$ and $\probabilitydistribution$ is defined as 
	\begin{equation*}
		\efn(h,\target,\probabilitydistribution) := \probabilitydistribution(\mu(h) \oplus \mu(\target)).
	\end{equation*}
\end{definition}

Since the learner does not know what $\probabilitydistribution$ and $\target$ are, the true error is not directly available to the learner.  A useful notion of error that can be calculated for a hypothesis   created by the learner is the {\em training error}, defined as follows.


\begin{definition}[Training Error]
	Given a hypothesis  $h\in\hypothesisSpace$ and a training set \setofexamples, the \emph{training error} of $h$ is defined as  the ratio of misclassified examples w.r.t. \setofexamples, in symbols $\efn_\Smc(h)$ is
	\begin{equation*}
		\resizebox{1\hsize}{!}{
			$\dfrac{|\{e\in \Smc_e\mid \mathsf{c}(e){=}1, e\not\in \mu(h)\}{\cup}\{e\in \Smc_e\mid \mathsf{c}(e){=}0, e\in\mu(h)\}|}{|\Smc|}.$
		}
	\end{equation*}
\end{definition}

We extract information from the black box models by posing membership queries, defined next.

\begin{definition}[Membership Query]
	Let $\bbC=(\examples, \hypothesisSpace, \mu)$ be a concept class. 
	Given a target $\target \in \hypothesisSpace$, let 
	${\sf MQ}_{\bbC,\target}$
	be the {\em oracle} that takes as input an example $\e$
	and outputs its classification  ${\sf c}(\e)$,
	where 
	$\e \in \examples$ is drawn according to the probability
	distribution $\probabilitydistribution$ and 
	${\sf c}(\e) = 1$, if $\e \in \mu(\target)$, and ${\sf c}(\e) = 0$, otherwise. A \emph{membership 
		query} is a call to the oracle ${\sf MQ}_{\bbC,\target}$.  
\end{definition}

One can create a training set using ${\sf MQ}_{\bbC,\target}$ by generating examples according to some probability distribution and asking ${\sf MQ}_{\bbC,\target}$ to classify the examples.

\subsection{Training Error Tolerance and Fidelity}\label{subsec:trainingerror}





A common assumption in PAC learning is the \emph{realizability assumption}, which basically says that the target belongs to the hypothesis space. 
When the target belongs to the hypothesis space, in theory we know there is a hypothesis that is consistent with the training data, that is, the training error is $0$.
This is a rather strong assumption in practice because the hypothesis is usually not fully consistent with the training set~\cite[Section 2.3]{MohriRostamizadehTalwalkar18}. 
On the other extreme, the notion of agnostic PAC learning completely removes the realizability assumption but  comes with other challenges.
Here we keep the realizability assumption but  allow the hypothesis to be inconsistent with the training set. We prove in Theorem~\ref{change_me} that one can allow the hypothesis to misclassify $k$ examples in a training set and give a bound on the minimal number of examples needed for PAC learnability based on $k$.





\begin{theorem}[Sample Size with Training Error]%
	\label{change_me}
	Let \(\hypothesisSpace\) be a finite hypothesis space from a concept class. Let $\delta,\epsilon \in (0, 1/2)$,   let $m,k\in\mathbb{N}$ be such that $m \geq k/\epsilon$ and 
	$$m \;\ge\; 
	\frac{1}{\epsilon}\!\left[\ln\!\left(\tfrac{|\Hmc|(k+1)\epsilon^{1-k}}{\delta}\right) + k\right] + k.
	$$
	Then, for any  $t\in\Hmc$  and for any distribution, $\probabilitydistribution$, 
	with
	probability of at least $1 -\delta$ over the choice of an i.i.d. sample \Smc of size $m$, we
	have that $
	\efn(h,t,\probabilitydistribution)\leq \epsilon$ for every
	$h\in\Hmc$ with $\efn_\Smc(h)\leq k/m$.
\end{theorem}

\begin{proof}
	The following proof is a non-trivial adaptation of the proof presented by~\citet{blumer86acm} (see also~\citet{UnderML2014}).
	Let \(\hsp\) be a family of boolean functions, \(t\) a boolean function, \(\probabilitydistribution\) a probability distribution, and \(\efn\) an error function. 
	We define, for \(\epsilon \in (0, 1)\), the following family of hypotheses, called ``bad hypotheses'':
	\begin{align*}
		\hsp_\epsilon \coloneqq \{h \in \hsp \mid \efn(h, t, \probabilitydistribution) > \epsilon\}.
	\end{align*}
	Given a sample $\Smc$, let $h^k_{\samples}$ be a fixed but arbitrary
	element of \(\hypothesisSpace\) such that 
	$e\in\mu(t)$ iff $e\in\mu(h^k_{\samples})$
	does not hold for at most $k$ elements $e$  in $\Smc$. 
	By the realizability 
	assumption,   $h^k_\Smc\in \hypothesisSpace$. 
	In other words,    we   have that 
	there is $h^k_\samples$ such that  \(|\Smc|\cdot\efn_{\samples}(h^k_\samples) \leq k\).
	Let \[M \coloneqq \{ \samples_e \mid \exists h \in \hsp_\epsilon \text{ s.t.\ } \efn_\samples(h) \leq \frac{k}{|\Smc|} \}.\] 

	\noindent   Since \(\efn(h^k_\samples, t, \probabilitydistribution) = \probabilitydistribution(\mu(h^k_\samples) \oplus \mu(t))\)
	we know that \(\efn(h^k_\samples, t,\probabilitydistribution) > \epsilon\) only if 
	\(h^k_\samples \in \hsp_\epsilon\). 
	Therefore, \(\{\samples_e \mid \efn(h^k_\samples, t, \probabilitydistribution) > \epsilon\} \subseteq M\). Hence, \(\probabilitydistribution^m(\{\samples_e \mid \efn(h^k_\samples, t, \probabilitydistribution) > \epsilon\}) \leq \probabilitydistribution^m(M)\).
	By the union bound,
	\begin{align}\label{res:upperbound}
			\probabilitydistribution^m(\{\samples_e \mid \efn(h^k_\samples, t, \probabilitydistribution) > \epsilon\}) \leq \nonumber\\
			\sum_{h \in \hsp_{\epsilon}} \probabilitydistribution^m(\{\samples_e \mid \efn_\Smc(h) \leq \frac{k}{|\Smc|}\}).
	\end{align}

	\noindent We now bound each summand of the second line of Inequality~\ref{res:upperbound}. Let us fix a bad hypothesis $h \in \hsp_\epsilon$. By i.i.d., 
	%
	\begin{align}\label{res:upperboundtwo}
		\probabilitydistribution^m(\{ \samples_e \mid  \efn_\Smc(h)\leq \frac{k}{|\Smc|}\})  =& \\ \nonumber\sum^k_{j=0}  (C^m_j\cdot  \probabilitydistribution(\Emc \backslash (\mu(h) \oplus \mu(t)))^{m-j}&\probabilitydistribution( \mu(h) \oplus \mu(t))^{j} 
		)
		%
	\end{align} 
	where $C^m_j$ is the number of ways to choose $j$ elements from a set of $m$ elements (considering these $m$ elements to be positions in the sequence of length $m$, possibly with repetitions). 
	Eq.~\ref{res:upperboundtwo} expands to the cases in which we have between $m-k$ and $m$ examples where the classification of the hypothesis matches with the classification of the target. 
	For each individual sampling of an element of the training set, 
	\begin{align*}
		\probabilitydistribution(\Emc \backslash (\mu(h) \oplus \mu(t)) 
		%
		=1 - \efn(h,t,\probabilitydistribution)\leq 1- \epsilon. 
	\end{align*}
	Using the inequality $1-x \leq e^{-x}$, we obtain for all $h\in \hypothesisSpace_\epsilon$:
	\begin{align*}
		\probabilitydistribution^m(\{\Smc_e\mid  \efn_\Smc(h)\leq \frac{k}{|\Smc|}\}) = \\ \sum^k_{j=0} (C^m_j\cdot (1 - \epsilon)^{m-j}\epsilon^j)
		\leq  \\ 
		\sum^k_{j=0} (C^m_j\cdot e^{- \epsilon (m-j)}\epsilon^j)
	\end{align*}
	By assumption, $k\leq m\cdot \epsilon$. Then, $k<m/2$ because 
	$\epsilon\in (0,1/2)$. This means that, for all $j\in [1,k]$, we have that $C^m_j\leq C^m_k$.
	Since $C^m_k  \leq (\frac{e \cdot m}{k})^k  $ (see \cref{lem:tech} in the appendix) and $\epsilon \geq \epsilon^j$ for all $j\in [0,k]$, we have that 
	\begin{multline}\label{buildingup}
		\sum^k_{j=0} (C^m_j\cdot e^{- \epsilon (m-j)}\epsilon^j)  \leq \\
		(k+1)\cdot (\frac{e \cdot m }{k})^k
		\cdot e^{- \epsilon (m-k)} \cdot \epsilon 
		\leq \delta
	\end{multline}
	By~\cref{res:upperbound} and \cref{buildingup}, and since $|\hsp_\epsilon|\leq |\Hmc|$ we have that 
	\begin{multline}
		\probabilitydistribution^m(\{\samples_e \mid \efn(h^k_\samples, t, \probabilitydistribution) > \epsilon\}) \leq \\
		|\hsp| \cdot (k+1)\cdot (\frac{e \cdot m }{k})^k
		\cdot e^{- \epsilon (m-k)} \cdot \epsilon 
		\leq \delta
	\end{multline}
	Taking logs and rearranging, we obtain  
	\begin{multline}
		\epsilon (m - k) \;\ge\; \ln\!\left(\tfrac{|\hsp|(k+1)}{\delta}\right) \;+\; k \,\ln\!\left(\tfrac{e m }{k} \right)+ \ln \epsilon.
	\end{multline}
	By assumption \( m \geq {k}/{\epsilon} \). We thus replace   the right occurrence of $m$ by ${k}/{\epsilon}$,  so
	\begin{multline*}    
		\epsilon (m - k) \;\ge\; \ln\!\left(\tfrac{|\hsp|(k+1)}{\delta}\right) \;+\; k \,\ln\!\left(\tfrac{e  }{\epsilon} \right)+ \ln \epsilon.
	\end{multline*}
	Then
	$\epsilon (m - k) \;\ge\; \ln\!\left(\tfrac{|\hsp|(k+1)}{\delta}\right) \;+
	k+ (1-k)\ln (\epsilon)$.
	Thus
	\[
	m \;\ge\; 
	\frac{1}{\epsilon}\!\left[\ln\!\left(\tfrac{|\Hmc|(k+1)}{\delta}\right) + k+ (1-k)\ln (\epsilon)\right] +     k.
	\]
	Finally,
	\[
	m \;\ge\; 
	\frac{1}{\epsilon}\!\left[\ln\!\left(\tfrac{|\Hmc|(k+1)\epsilon^{1-k}}{\delta}\right) + k\right] +    k.
	\]
	This concludes the proof of the theorem.
\end{proof}




We now relate the PAC notion with the notion of fidelity, classically studied within the explainable AI literature.
\begin{definition}[Fidelity and Probabilistic Fidelity]
	The \emph{fidelity} of a post-hoc explanation $h$
	w.r.t. a black box model $t$ and a set of examples $\examples$ is defined as:
	\[{\sf Fid}^\examples(h,t)=\frac{|\{e\in \examples\mid e\in \mu(h)\text{ iff }e\in \mu(t)\}|}{|\examples|}.
	\]
	The \emph{probabilistic fidelity} of a post-hoc explanation $h$
	w.r.t. a black box model $t$, a set of examples $\examples$, and a probability distribution $\probabilitydistribution$ is: 
	\[{\sf Fid}^\examples_\probabilitydistribution(h,t)=\probabilitydistribution(\{e\in \examples\mid e\in \mu(h)\text{ iff }e\in \mu(t)\}).
	\]
	That is,  the probabilistic fidelity is $1-\probabilitydistribution(\mu(h) \oplus \mu(t))$.
\end{definition}

\begin{corollary}[Probabilistic Fidelity]%
	Let \(\hypothesisSpace\) be a finite hypothesis space from a hypothesis class. Let $\delta,\epsilon \in (0, 1/2)$ 
	and  $m,k\in\mathbb{N}$, $m\geq k/\epsilon$. If
	$$m \geq
	\frac{1}{\epsilon}\!\left[\ln\!\left(\tfrac{|\Hmc|(k+1)\epsilon^{1-k}}{\delta}\right) + k\right] + k.
	$$
	then,  for any distribution $\probabilitydistribution$, and for  any $t\in\Hmc$, 
	with
	probability of at least $1 -\delta$ over the choice of an i.i.d. sample $\Smc$ of size $m$, we
	have that $ 
	\ {\sf Fid}^\examples_\probabilitydistribution(h,t) \geq 1-\epsilon$ for every
	$h\in\Hmc$ with $\efn_\Smc(h)\leq k/m$. 
\end{corollary}
Fidelity is more relevant   than accuracy 
when the goal is to 
approximate   black box models (not ground truth).


	
	



\section{ PAC Decision Trees  }\label{sec:decisiontrees}
\newcommand{\targett}{\ensuremath{T_\star}\xspace}
We first provide basic notions for   defining  decision trees (\cref{sec:decision_tree}).
Then, we present 
a  version of a decision tree algorithm, called 
\(\TopDown\), introduced by~\citeauthor{Kearns1999}~(\citeyear{Kearns1999}) (\cref{subsec:boosting})
and the \Trepan~algorithm (\cref{subsec:trepan}) for building decision trees with PAC guarantees.


\subsection{Decision Trees }\label{sec:decision_tree}

We  provide formal definitions for candidate splits, constraints, trees and  decision trees. 

\begin{definition}[Candidate Split \& Constraint]\label{def:constraint} Let \Fmc be a set of \emph{features}.
	Each feature \feature is associated with a set ${\sf values}(\feature)$ of possible values.
	A \emph{candidate split} is an expression of the form $\feature \ocirc v$ where $f\in \Fmc$, $\ocirc\in \{=,<,\leq,\geq, >\}$, 
	and $v\in {\sf values}(\feature)$.
	We define
	\emph{constraints}   by induction.
	Every candidate split is a constraint.
	If $\phi$ and $\psi$ are constraints  
	then $(\phi \vee\psi)$ and $(\phi \wedge\psi)$
	are constraints.
	Moreover, if $\phi$ is a constraint then $\neg(\phi)$ is 
	also a constraint. We may 
	write $\overline{\phi}$ 
	as a short hand for $\neg(\phi)$.
	When dealing with constraints, we may treat sets and the conjunction of its elements interchangeably (e.g. $a\wedge b$ as $\{a,b\}$ and vice-versa). 
	Moreover, we use $0$
	as an abbreviation for a contradiction $\psi\wedge\neg(\psi)$, with $\psi$ being a candidate split,
	and we write $1$  for the negation of $0$.
\end{definition}

\newcommand{\constraint}{\ensuremath{\psi}\xspace}
\begin{definition}[Tree]\label{def:tree}
	A (binary) \emph{tree} is a set $\tau$ containing $\emptyset$ and  finite binary sequences
	closed under prefix
	(i.e., if e.g. $110\in  \tau $ then
	$11\in  \tau$ and $1\in  \tau $). 
	The elements of $\tau$ are called \emph{nodes}
	and the node $\emptyset$ is called the \emph{root} of $\tau$.  
	The parent of a (non-root) node $n\in \tau$
	is 
	the result of removing the last element
	of $n$. E.g.,  $11$ is the parent of $110$.  
\end{definition}
\newcommand{\decisiontree}{\ensuremath{T}\xspace}

\begin{definition}[Decision Tree]\label{def:decisiontree}
	A (binary) \emph{decision tree}
	is a relation  that maps each 
	node in a tree $\tau$
	to a (possibly empty) set of constraints. 
	We  refer to this relation
	as a set of pairs $(n,\constraint)$ where 
	$n$ is a node in $\tau$ and \constraint is a constraint. 
	The \emph{class} of a node $n$, denoted ${\sf class}(n)$,
	is the last bit of  the sequence $n$. E.g.,
	if $n=110$ then ${\sf class}(n)=0$. 
	%
\end{definition}

Whenever we speak of the nodes of a decision tree we mean 
the domain of the relation.
The nodes of a decision tree 
can be \emph{internal nodes} or \emph{leaf nodes}.
The former are nodes which are parents of other nodes, while the latter are the remaining nodes, that is, those without children. We denote by  
$\mathsf{size\_of}(\decisiontree)$ the number of internal nodes of a decision tree $\decisiontree$.  Also, we write 
$\leaves{\decisiontree}$ for the
set 
of leaf nodes in  \decisiontree.
We write $T(l,\constraint)$ for the tree that  results of transforming a leaf $l$
in a tree $T$ into an internal node (with two children) and associating the constraint \constraint with 
this internal node. In other words, 
$T(l,\constraint):= (T\setminus \{(l,\emptyset)\})\cup \{(l,\constraint),(l0,\emptyset),(l1,\emptyset)\}$.

Given a node $n$ in a decision tree \tree, we write  $\mathsf{constr}_\tree(n)$ for the set of constraints associated with   $n$ in $\tree$ (that is $\{\psi\mid (n,\psi)\in \tree\}$). Also,  we define a function $\mathsf{constr}^\ast_\tree$ that maps each node to all constraints needed to be satisfied to reach the node. We define such function inductively as follows. 
First we set
$\mathsf{constr}^\ast_\tree(\emptyset):=\emptyset$.
Now, 
we set $\mathsf{constr}^\ast_\tree(n0):=\overline{\mathsf{constr}_\tree(n)}\cup \mathsf{constr}^\ast_\tree(n)$ and
$\mathsf{constr}^\ast_\tree(n1):= {\mathsf{constr}_\tree(n)}\cup \mathsf{constr}^\ast_\tree(n)$.

\begin{definition}[Tabular Examples \& Constraint Satisfaction]\label{def:tabularexamples}
	A \emph{tabular example} for a set of features $\Fmc$
	is a 
	tuple $(v_1,\ldots,v_n)$ with each
	$v_i\in {\sf values}(\feature_i)$ being a value for a feature $\feature_i\in \Fmc$ (assume a fixed but arbitrary order for the features in $\Fmc$).
	Then,
	$ (v_1,\ldots,v_n)$ \emph{satisfies} 
	\begin{itemize}
		\item a candidate split 
		$f_i \ocirc v$ if $v_i\ocirc v$, with $\ocirc\in \{=,<,\leq,\geq, >\}$;
		\item a constraint $\neg (\constraint)$
		if it does not satisfy the constraint $\constraint$;
		\item a constraint 
		$(\constraint\wedge \psi)$ if it satisfies 
		$\constraint$ and $\psi$;
		\item a constraint 
		$(\constraint\vee \psi)$ if it satisfies 
		$\constraint$ or $\psi$.
	\end{itemize}
	Given a tabular example $x$ and a constraint $\psi$, we may write $x\models \psi$ to express that $x$ satisfies $\psi$. Moreover, 
	we say that a decision tree \tree \emph{classifies} a tabular example 
	$\tabularexample$ as \emph{positive}, written $T(\tabularexample)=1$, if there is a leaf node $n1$ such that 
	$\tabularexample\models \mathsf{constraints}_\tree(n1)$. Otherwise,   \tree \emph{classifies} the tabular example 
	$\tabularexample$ as \emph{negative}.
\end{definition}

\subsection{The TopDown Decision Tree Algorithm}\label{subsec:boosting}

We     define the concept class of (binary) decision trees.
\begin{definition}\label{def:dtconceptclass}
The \emph{concept class} of (binary) decision trees \(\bbCDT\) is a triple 
(\examples, \hypothesisSpaceDT, $\mu$) where  
$\examples$ is a set of tabular examples (according to Definition~\ref{def:tabularexamples}), 
\hypothesisSpaceDT is a set of (binary) decision trees,
and 
$\mu$ is a function that maps each element $\tree$ of \hypothesisSpaceDT to
$\bigcup_{n1\in \leaves{\tree}}\{x\in\examples\mid x\models {\sf constr}^\ast_\tree(n1) \}$.
\end{definition}

Let \tree be a decision tree (\cref{def:decisiontree}). For each $l \in \mathsf{leaves}(\tree)$, 
$\probabilitydistribution(   \{ x\in\examples | x \models {\sf constr}^\ast_\tree(l) \})$
%
%
is the probability that an example $x\in\examples$ (drawn according to $\probabilitydistribution$) reaches $l$ in \tree. 
The \emph{error of a leaf} $l$   is given by   the probability that $l$ wrongly classifies $x$. 
\emph{The error of a decision tree} \tree is 
the sum of the error of 
all the leaves in   \tree. 
We write $\targett$ 
for the 
\emph{target decision tree}
that represents the black box binary classifier 
we attempt to approximate. 
 \begin{restatable}{proposition}{proptech}
The true error $\efn(\tree,\targett,\probabilitydistribution)$ of a decision tree \tree  is 
	\begin{align*}
		&  \sum_{n0 \in \leaves{\tree}}  \probabilitydistribution(   \{ x | x \models {\sf constr}^\ast_\tree(n0) \}  \cap {\mu(\targett)}  ) + \\
		&  \sum_{n1 \in \leaves{\tree}}  \probabilitydistribution(   \{ x | x \models {\sf constr}^\ast_\tree(n1) \}  \cap \overline{\mu(\targett)} ).
	\end{align*}
\end{restatable}

\begin{proof}
The proof is at \arxivurl.
\end{proof}

\begin{algorithm}[t]
\begin{algorithmic}[1]
\small
\STATE $T:=\{(\emptyset,\emptyset)\}$
\WHILE{$\mathsf{size\_of}(T) < \mathsf{size\_limit}$}
\STATE $\Delta_{best} \gets 0$ 
\FOR{each  $(l,\constraint)\in {\sf leaves}(T)\times \setofconstraints$}\label{ln:for}
\STATE $\Delta \gets \efn(\tree,\targett,D)- \efn(\tree(l,\constraint),\targett,D) $
\IF{$\Delta \geq \Delta_{best}$}
\STATE $\Delta_{best} \gets \Delta$
\STATE $(l_{best},\constraint_{best}) \gets (l,\constraint)$
\ENDIF
\ENDFOR
\STATE $T \gets T \cup \{(l_{best}, \constraint_{best})\}$
\ENDWHILE
\STATE Return ($T$)
\end{algorithmic}
\caption{TopDown($\efn(\cdot,\targett,\probabilitydistribution)$, 
$\mathsf{size\_limit}$,
$\setofconstraints$) . 
}
\label{algo:topdown}
\end{algorithm}


\noindent The  \(\TopDown\) algorithm (\cref{algo:topdown}) starts with a root node and builds a decision tree by creating new leaf nodes at each iteration of the `while loop'. These leaves may become parents at the next iteration, and so on.  
The  `for loop' iterates over the space of 
possible leaves and constraints. The idea is to
find a leaf and constraint  in a way that the splitting 
minimizes the error of the tree.

Algorithm~\ref{algo:topdown} is only useful for theoretical purposes since in practice we do not have access to the function  $\efn(\cdot)$ and, furthermore, it would be computationally expensive to make an exhaustive search on the constraints, as done in 
Line~\ref{ln:for}. So in the next section we adapt a decision tree algorithm to use the PAC framework.

\newcommand{\Cmc}{\ensuremath{\Phi}\xspace}
\subsection{The \trepan  Algorithm}
\label{subsec:trepan}
\trepan is a tree induction algorithm that   extracts decision trees from binary classifiers, seen as oracles.
The original motivation behind the development of \trepan is to approximate a neural network by means of a symbolic structure that is more interpretable than a neural network classification model while giving PAC guarantees. 
This approach has also its roots in the context of a wider interest in knowledge extraction from neural networks (see  Introduction).  
\trepan differs from conventional inductive learning algorithms such as CART and C4.5 
because it uses the PAC framework  to estimate the amount of training data and internal nodes in the resulting decision tree. It uses a membership oracle  to classify examples used for training.
Here we consider the binary entropy as the splitting criterion, that is,
\(G(q) = -q \log(q) - (1-q)\log(1-q)\) and $q\in [0,1]$.
The \emph{best (binary) split}, denoted $\bestsplit$, is the candidate split with the lowest entropy  w.r.t a set of samples $\mathbb{S}$.

The pseudo-code for \trepan is shown in Algorithm~\ref{algo:trepan}.  
In more details, the \trepan algorithm works  as follows. It takes $5$ inputs, namely, an oracle ${\sf MQ}_{\bbC,\targett}$ for creating the sample, 
the maximal number of internal nodes (called  $\mathsf{size\_limit}$), an upper bound  $k$ on the number of misclassified examples,
a set of candidate splits $\Cmc$, and
the size of the training set (called $\mathsf{training\_size}$).
At the beginning, 
the algorithm creates an empty node and pushes it to the 
queue (Line~\ref{ln:push}).  
The queue is used to decide which nodes should be evaluated  (Line~\ref{ln:pop}). 
The algorithm calls the oracle ${\sf MQ}_{\bbC,\targett}$ to create a training set (Line~\ref{ln:createtrainingset}). 

\trepan stops the tree extraction process when one of three criteria is met: no more nodes need to be further expanded (the queue is empty), 
a predefined limit of the tree size is reached, or the  training error is below a predefined bound (Line~\ref{ln:while}). In Line~\ref{ln:cutoff}, by the ``misclassification of a node $ni$'' we mean the number of classified examples in $\mathbb{S}$ that reach $n$, but whose class differs from $i$.









\begin{algorithm}[t]
\begin{algorithmic}[1]
\small
\STATE Queue $Q \gets \emptyset$ 
\STATE Decision tree $T \gets \{(\emptyset, \emptyset)\}$ 
\STATE Push the root node $\emptyset$ into $Q$ \label{ln:push}
\STATE Let $\mathbb{S}$ be a  set created with ${\sf training\_size}$  calls to 
${\sf MQ}_{\bbC,\targett}$
\label{ln:createtrainingset}
\WHILE{$Q\neq \emptyset$, $\mathsf{size\_of}(T) < \mathsf{size\_limit}$, and \\
	$\quad\quad\quad$ $\quad\quad\quad$ $\quad\quad\quad$ $\quad\quad$   ${\sf error}_{\mathbb{S}}(T)> k/{\sf training\_size}$\label{ln:while}}
\STATE Pop   $n$ from   $Q$ \label{ln:pop} 
\STATE Use   
$\mathbb{S}$
to find $\bestsplit\in\Cmc$ based on binary entropy
\STATE $T:=(T\setminus\{(n,\emptyset)\})\cup\{(n,{\bestsplit}),(n 0,\emptyset),(n 1,\emptyset) \}$
\FOR{$i\in \{0,1\}$}
%
%
\IF{$ni$ misclassification is  $> k/({\sf size\_limit}+1)$}\label{ln:cutoff}
\STATE Push   $n i$ into $Q$  
\ENDIF
\ENDFOR
\ENDWHILE
\STATE Return ($T$)
\end{algorithmic}
\caption{\trepan{(${\sf MQ}_{\bbC,\targett}$, 
	$\mathsf{size\_limit}$, $k$, $\Cmc$, ${\sf training\_size}$)}
	}
	\label{algo:trepan}
\end{algorithm}











\begin{corollary}
Let ${\sf MQ}_{\bbC,\targett}$ be the membership oracle for the concept class \bbC and the target \targett and let \Cmc be the set of candidate splits induced by \bbC.
Let $T$ be the output of a run of \trepan(${\sf MQ}_{\bbC,\targett}$, 
$\mathsf{size\_limit}$, $k$, $\Cmc$, ${\sf training\_size}$) and let 
$\epsilon,\delta\in (0,1/2)$.
With probability greater than \(1 - \delta\), 
we have that \(\mathsf{error}(\tree,\targett,\probabilitydistribution)\)
%
is  smaller than \(\epsilon\), where 
$\probabilitydistribution$ is the probability distribution used to generate the examples given as input to the membership oracle ${\sf MQ}_{\bbC,\targett}$ if (i)
${\sf training\_size}$ is
as in Theorem~\ref{change_me} and the number of misclassified examples of $T$ w.r.t. the training set is bounded by $k$; or (ii) $\mathsf{size\_limit}$ (the number of internal nodes) 
and ${\sf training\_size}$   are
as in Theorem~1 by~\cite{Kearns1999} and $k=0$. 
%
\end{corollary}
\begin{proof}
This result is a consequence of 
Theorem~\ref{change_me} and
in Theorem 1 by~\citeauthor{Kearns1999}~(\citeyear{Kearns1999})). 
%
%
%
%
%
%
%
%
%
%
\end{proof}

{




\section{Experiments}\label{sec:experiments}


We now describe our experiments using \trepac. In order to showcase the performance of   
\trepan as a tool for explaining  the behaviour  black box models, we consider the recent case study presented by~\citeauthor{BLUM2023109026}~\citeyearpar{BLUM2023109026}, where the authors extract rules indicating occupational gender biases in 
BERT-based language models. 
In the mentioned reference, the authors
use an adaptation of Angluin's Horn algorithm to extract rules. The authors \emph{do not} provide PAC guarantees for their results (although this could have been obtained in their work by running the algorithm without limiting the number of simulated equivalence queries). 
Here we extract decision trees and analyse the results in light of the PAC framework.
Although BERT-based models~\cite{DBLP:conf/naacl/DevlinCLT19,DBLP:journals/corr/abs-1907-11692} are not binary classifiers, the case study focuses on pronoun prediction (`she' or `he').\footnote{The non-binary pronoun `they' did not receive   significant prediction values by the models, see Table~2 in \cite{BLUM2023109026}.} We describe the case study by~\citet{BLUM2023109026} (\cref{subsec:studycs})
and also our experimental results (\cref{subsec:results}).
\subsection{The Occupational Gender Bias Case Study}\label{subsec:studycs}
In the case study by~\citet
{BLUM2023109026}, the authors consider $4$
language models:  BERT-base-cased, BERT-large-cased~\cite{DBLP:conf/naacl/DevlinCLT19}, RoBERTa-base, and RoBERTa-large~\cite{DBLP:journals/corr/abs-1907-11692}. These language models are trained to fill up
a blank part (called `mask') of a sentence with a token.
The authors of the case study use the following   kinds of features  
to create sentences\footnote{They also include `unknown value' features, relevant to the Horn algorithm but unnecessary for our work.}:
\begin{itemize}
\item \textbf{birth period:} before 1875, between 1875 and 1925, between 1925 and 1951, between 1951 and 1970, after 1970; 
\item \textbf{location:} North America, Africa, Europe, Asia, South America, Oceania, Eurasia, Americas, Australia;
\item \textbf{occupation:} nurse, fashion designer, dancer, footballer, industrialist, boxer, singer, violinist.
\end{itemize}

In total, this gives  $22$ features which can be used to create sentences with one value for each kind of feature. 
The sentences are of the form:
\texttt{<mask> was born [birth period] in [location] and is a/an [occupation]}, where the task of a BERT-based model is to predict \texttt{<mask>}.
For instance, one possible sentence in this template would be: \texttt{<mask> was born after 1970 in Africa and is a singer.}
Since there are $5$ birth periods, $9$ locations, and $8$ occupations in the study, one can form $5\cdot 9\cdot 8=360$ sentences using the template.
Even though the model could fill the mask with any token, the authors indicate  that the most likely completions are one of the two pronouns `she' or `he', which works as a binary classification task.
Each sentence is mapped to a
binary vector using a lookup table (see at \arxivurl).
The binary vector is a one-hot encoding of the features. Our sentence in the template above would be mapped to the binary vector:
\[ [0,0,0,0,1,0,1,0,0,0,0,0,0,0,0,0,0,0,0,0,1,0 ]\]
If the BERT-based model predicts the pronoun `she' then this would correspond to classifying the example with $0$ and if it would predict `he' then the classification would be $1$.

\begin{table}[!t]
\begingroup
\setlength{\tabcolsep}{4pt}
\centering
\small	
\begin{tabular}{|c||c ||c| c|c|c|}
\hline
$c$& ${ n}$ &   $m:k=0$ &  $m:k=5$ & $m:k=10$ & $m:k=15$\\
\hline
$0.04$ & $3$ & $124$ & $204$ & $277$ & $349$\\
$0.06$  & $6$ & $201$ & $280$ & $353$ & $425$\\
$0.08$  & $10$ & $257$ & $336$ & $409$ & $481$ \\
$0.1$  & $18$ & $321$ & $401$ & $474$ & $546$\\
\hline 
\end{tabular}
\caption{Sample size $m$  calculated c.f.  \cref{change_me} with $\delta = 0.1$, $\epsilon=0.2$, and at most $k$   misclassified examples. Values are rounded and calculated with the size of the hypothesis space   being the number of all decision trees with ${n}$ internal nodes. 
$n$ is estimated c.f.~Theorem 1~\cite{Kearns1999} 
with constant $c = [0.04, 0.06, 0.08, 0.1]$ and $\gamma=0.5-\epsilon$.
}
\label{tab:sample}
\endgroup
\end{table}

\subsection{Results}\label{subsec:results}
We now present   our experimental results.
The idea behind our experiments is to 
(i)     check the \emph{feasibility} of the case study~\cite{BLUM2023109026} in our setting, which gives PAC guarantees, 
(ii) check whether we obtain results that are \emph{consistent} with their results, and also 
(iii) check whether the \emph{format} of decision trees provides additional useful information.
Experiments were run on a MacBook Pro M2 16GB. Each experiment was run $10$ times, we present average values.

\begin{figure}[!t]
\captionsetup[subfigure]{labelformat=empty}
\centering
\subfloat[]
{
	\includegraphics[width=0.49\columnwidth]{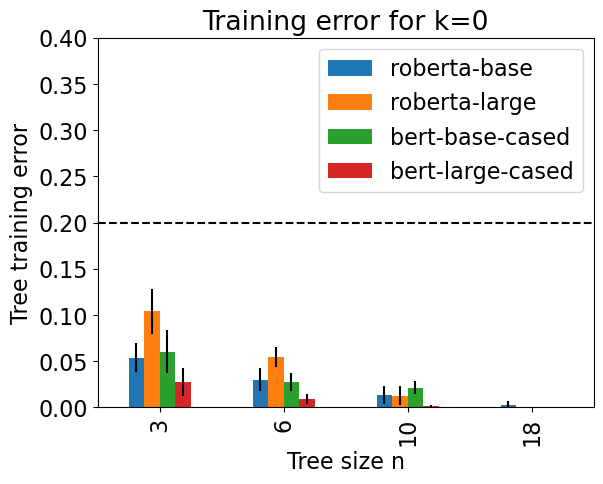}
	\includegraphics[width=0.49\columnwidth]{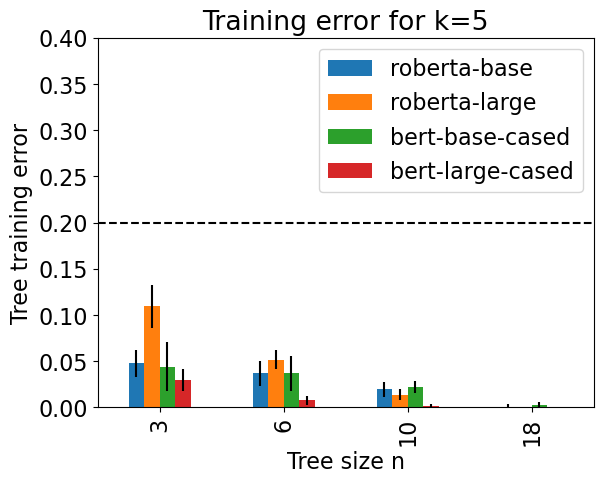}
	
}
\qquad
\subfloat[]
{
	\includegraphics[width=0.49\columnwidth]
	{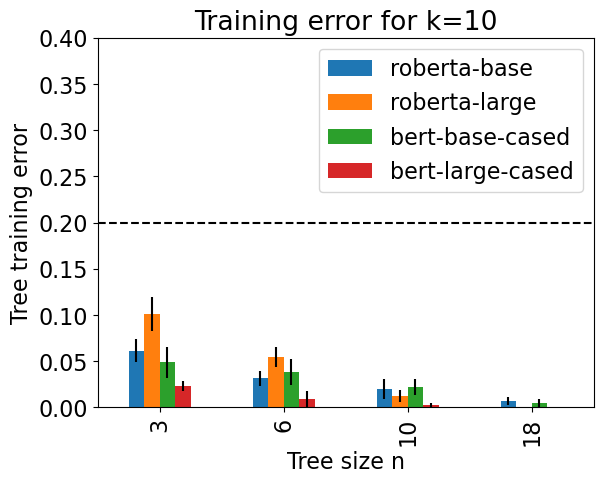}
	\includegraphics[width=0.49\columnwidth]
	{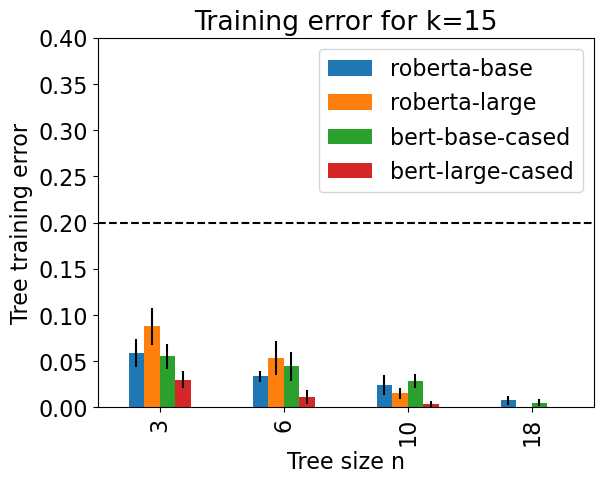}
	
}
\caption{\trepac training error for $k=[0,5,10,15]$ and $n=[3,6,10,18]$. 
	Increasing the number of internal nodes reduces the training error. The horizontal dotted line correspond to $\epsilon=0.2$. 
	See also \cref{tab:sample}.
}
\label{fig:training_error}
\end{figure}

\begin{figure}[!t]
\captionsetup[subfigure]{labelformat=empty}
\centering
	\subfloat[]
	{
		\includegraphics[width=0.49\columnwidth]{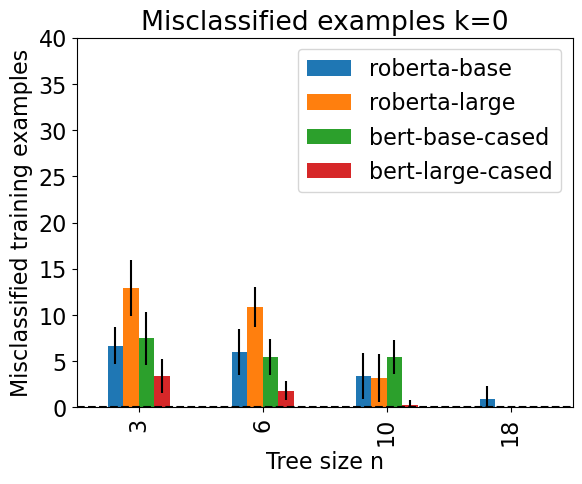}
		\includegraphics[width=0.49\columnwidth]{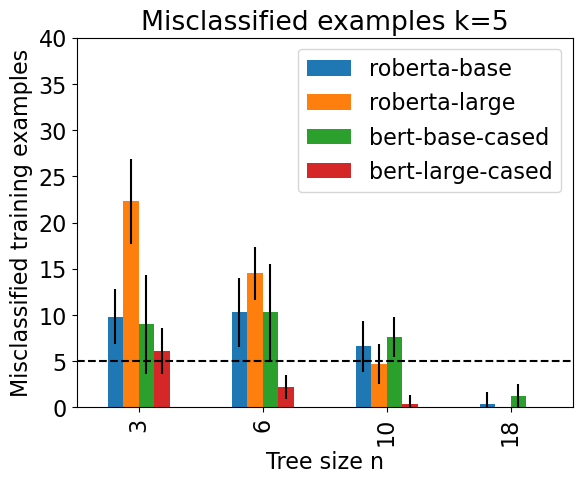}
		
	}
	\qquad
	\subfloat[]
	{   \includegraphics[width=0.49\columnwidth]
		{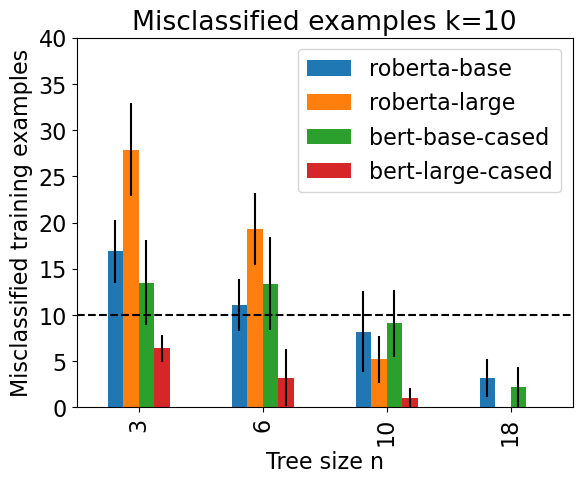}
		\includegraphics[width=0.49\columnwidth]
		{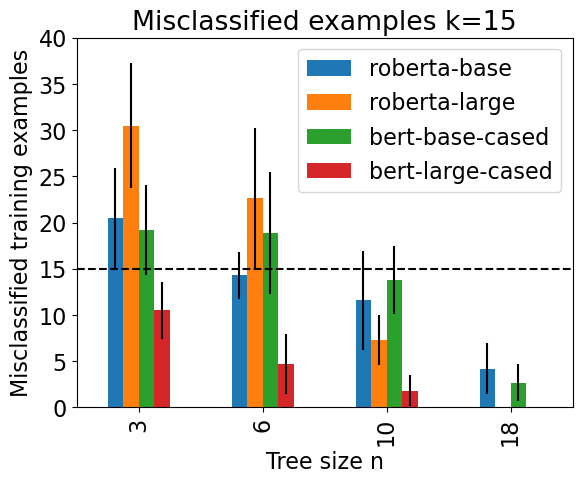}
		
	}
	\caption{\trepac misclassified training examples for $k=[0,5,10,15]$ and $n=[3,6,10,18]$. The number of misclassified examples  is at most $k$ (horizontal dotted line) when an appropriate tree size is chosen. 
	}
	\label{fig:mis_examples_training}
\end{figure}

\paragraph{Feasibility} We ran several experiments to check the practical feasibility of the approach using the PAC framework. The  binary data representation is suitable for the case of learning surrogate decision trees from the BERT-based language models.
We use \cref{change_me} to calculate $m$ considering  the hypothesis space to be the set of all
decision trees with $n$ internal nodes.
The size of the hypothesis space was constrained to the number of all decision trees with $n$ internal nodes, in particular $|n|^{|F|}$, where $|F|$ is number of features which in our case is $22$. The corresponding sample sizes ($m$) for values of $k=[0,5,10,15]$ and $n = [3, 6, 10, 18]$ are shown in Table~\ref{tab:sample}. 
The range values for $k$ and $n$ were chosen so that the number of training examples is low in comparison with the total number of examples. 
When the number of accepted errors $k$ is larger, the number of hypotheses which can be used increases. This also raises the chance of selecting a bad hypothesis. To compensate for this, a larger number of consistent examples is required.  
%
For each language model we extracted $16$ decision trees varying $n$ and $k$ according to Table~\ref{tab:sample}.  
For each decision tree we measured the training error (Figure~\ref{fig:training_error}), as well as 
the number of misclassified training examples (Figure~\ref{fig:mis_examples_training}).  

As expected, both the training error and the number of misclassified training examples decrease whenever the size of the tree increases.  Figure~\ref{fig:training_error} shows the training error of the BERT-based language models. For all models, this error is less than the upper bound $\epsilon=0.2$, chosen for determining the number of examples needed for PAC learnability based on $k$. This result is consistent with our \cref{change_me}. Figure~\ref{fig:mis_examples_training} shows the sensitivity of the results w.r.t. the number of accepted errors $k$. The numbers of samples $m$ calculated in Table~\ref{tab:sample} ensure that the number of misclassified training examples is at most $k=0$ when the tree size is at least $n = 18$ for all models. For other values $k=[5,10,15]$, $m$ ensures that the number of misclassified training examples is at most $k$ when an appropriate tree size is chosen. To have at most $k=5$ misclassified examples, the tree size needs to have at least $n=10$ internal nodes. For higher values of $k$, trees with fewer nodes suffice. These results support the practical applicability of \cref{change_me}. In our experiments, the training errors of the decision trees extracted from RoBERTa-base and RoBERTa-large are higher than those from BERT-base and BERT-large. This suggests that models trained on larger corpora (which is the case for the RoBERTa models) may be more complex,
which translates into more internal nodes.  
We provide additional supporting information regarding the error at \arxivurl.
Crucially, the time for running the algorithm is negligible. The time for generating training sets using BERT-based models varied from approximately $85$ to $510$ seconds  (more at \arxivurl). 


\begin{figure}[!t]
\centering
\includegraphics[width=0.31\textwidth]{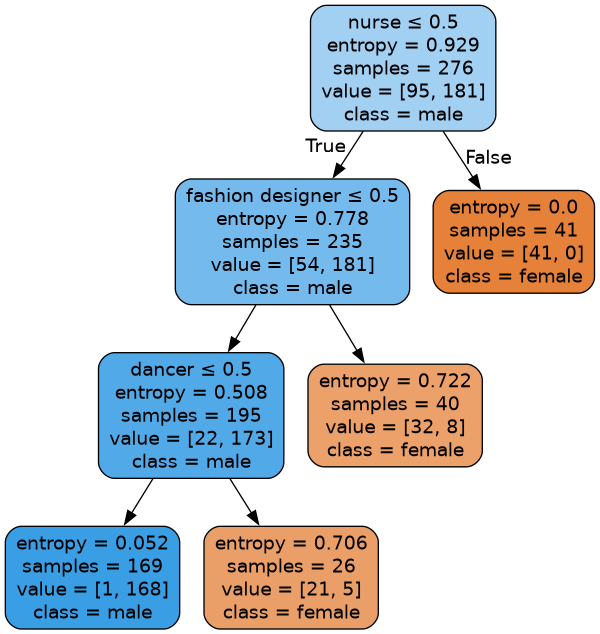}

\caption{Surrogate tree for RoBERTa-base. The tree  is extracted with $n=3$ and $m=277$   ($k=10$), see \cref{tab:sample}.}
\label{fig:enter-label}
\end{figure}

\paragraph{Consistency and Format}
Our results are consistent with those obtained by~\citeauthor{BLUM2023109026}~\citeyearpar{BLUM2023109026}, indicating occupational gender bias in these models (see~\cref{fig:enter-label}).
Regarding the format,
the decision trees  facilitate the visualization of 
which features are most relevant for pronoun prediction, with the occupations being the most relevant ones. The decision trees can also  be used to extract rules. 
We report the frequency of common rules in \cref{tbl:rules}. The `nurse' occupation appears as the most biased occupation towards female.

\begin{table}[t]
\centering
\begin{tabular}{ll}
	Rule & Frequency \\
	\hline \hline
	nurse $\rightarrow$ female  & $0.50$ \\
	industrialist $\rightarrow$ male & $0.18$ \\
	footballer $\rightarrow$ male & $0.16$ \\
	boxer $\rightarrow$ male & $0.12$ \\
	fashion\_designer $\rightarrow$ female & $0.02$ \\
	dancer $\rightarrow$ female & $0.02$ \\
	%
\end{tabular}
\caption{Most frequent rules extracted from the surrogate decision trees of the BERT-based  models. 
}
\label{tbl:rules}
\end{table}





\section{Conclusion and Future Work}\label{sec:conclusion}

We motivate and study   PAC guarantees for decision trees extracted from black box binary classifiers. 
Since the training error is rarely zero, we provide a non-trivial proof for estimating the sample size for finite hypothesis when the training error is not zero. Our proof is applicable to any concept class with finite hypothesis space.
We then  formalize the relevant definitions for decision trees.
Finally, we perform experimental results with decision trees on BERT-based models and analyse the results in light of the PAC framework. 

In future work, it would be interesting to extend the experiments with  other case studies and consider multi-classification tasks. Also, one could  study a sampling strategy that satisfies the constraints of the leaves, as in Trepan~\cite{Craven1995}, within the PAC framework. This seems promising from a practical point of view. However, the analysis would be complex because randomly selecting examples that satisfy the leaves would change the sample space, necessitating an investigation into how the PAC framework could be adapted for this scenario.

\section*{Acknowledgments}
We thank the anonymous reviewers for their comments. Ozaki is supported by the Research Council of Norway, project (316022).  This work was also supported by the Research Council of Norway, Integreat - Norwegian Centre for knowledge-driven machine learning (332645). Confalonieri acknowledges funding support from the ‘NeuroXAI’ project (BIRD231830).

\bibliographystyle{plainnat}
\bibliography{aaai25.bib}


\appendix

\section{Proof of \cref{change_me}}
\begin{lemma}\label{lem:tech}
For all $m\in \mathbb{N}$ and $j\in [1,m]$, we have that    $\binom{m}{j} \leq (\frac{e \cdot m}{j})^j$. 
\end{lemma}
\begin{proof}
This lemma is based on Stirling's formula~\cite{stirlingprobbook}. 
We first note the following.
\begin{multline}
	\label{eq:prod}
	\binom{m}{j} \;=\; \frac{m!}{j!(m-j)!}
	\;=\; \\ \frac{m(m-1)\cdots(m-j+1)(m-j)!}{j!(m-j)!}\;=\; \\ \frac{m(m-1)\cdots(m-j+1)}{j!}
	\;\le\; \frac{m^j}{j!}
\end{multline}
\noindent
Since~\cite[page 52, Eq. (9.5)]{stirlingprobbook}:
\begin{equation*}
	\ln (j!)> j \ \ln (j) - j,
\end{equation*}
we have
\begin{equation*}
	e^{\ln (j!)}> e^{j \ \ln (j) - j}.
\end{equation*}
So
\begin{equation}
	\label{eq:stirling}
	j!> j^j /e^j = (j/e)^j.
\end{equation}
Combining \cref{eq:prod} and \cref{eq:stirling} we have that
\begin{align*}
	\binom{m}{j}
	\;\le\; \frac{m^j}{j!}
	\;\le\; \frac{m^j}{(j/e)^j}
	\;=\; \Big(\frac{e\,m}{j}\Big)^{\!j}.
\end{align*}
\end{proof}

\section{Proof of Proposition 16}\label{sec:proofprop16}

The true error is the probability of misclassification, which happens when the classification of an example does not match with the classification of the decision tree. 
\proptech*
\begin{proof}
By Definition~\ref{def:trueerror}, 
\[\efn(\tree,\target,\probabilitydistribution)=\probabilitydistribution(\mu(\target)\oplus \mu(\tree)).\]
We have that $x\in\mu(\target)\oplus \mu(\tree)$ iff either
$x\in\mu(\target)\setminus \mu(\tree)$ or 
$x\in\mu(\tree)\setminus \mu(\target)$. 
The latter 
is equal to
$\mu(\tree)\cap \overline{\mu(\target)}$. 
Also, by \cref{def:dtconceptclass},
\[\mu(\tree)=\bigcup_{n1\in \leaves{\tree}}\{x\in\examples\mid x\models {\sf constr}^\ast_\tree(n1) \}.\]
Since each example reaches exactly one leaf node,
\[\probabilitydistribution(\mu(\tree))=\sum_{n1 \in \leaves{\tree}} \probabilitydistribution( \{x\in\examples\mid x\models {\sf constr}^\ast_\tree(n1) \}).\]
So $\probabilitydistribution(\mu(\tree)\cap \overline{\mu(\target)})$ is equal to 
\[\sum_{n1 \in \leaves{\tree}}  \probabilitydistribution( \{   \tabularexample | \tabularexample \models {\sf constr}^\ast_\tree(n1) \}  \cap \overline{\mu(\target)} )  .\]
Similarly $\probabilitydistribution(\mu(\target)\setminus \mu(\tree))= \probabilitydistribution(\overline{\mu(\tree)}\cap\mu(\target))$ and, by the same argument as above, it is equal to 
\[\sum_{n0 \in \leaves{\tree}}  \probabilitydistribution(   \{ x | x \models {\sf constr}^\ast_\tree(n0) \}  \cap {\mu(\target)}  ).\]
\end{proof}

\section{Lookup Table}\label{sec:lookup}
\cref{tbl:lookup} contains the mapping between features and vector positions, as defined by the authors of the occupational gender bias case study~\cite
{BLUM2023109026}.
\begin{table}[!t]
\caption{Lookup table. For each kind of feature at most one of the positions can be chosen (setting its value to $1$).}
\centering
\begin{tabular}{ll}
	\hline \hline
	position & value \\
	\hline 
	time period & $~$\\
	0 & before 1875\\
	1 & between 1825 and 1925\\
	2 & between 1925 and 1951\\
	3 & between 1951 and 1970\\
	4 & after 1970\\
	\hline
	continent & $~$\\
	5 & North America\\
	6 & Africa\\
	7 & Europe\\
	8 & Asia\\
	9 & South America\\
	10 & Oceania\\
	11 & Eurasia\\
	12 & Americas\\
	13 & Australia\\
	\hline 
	occupation & $~$\\
	14 & nurse\\
	15 & fashion designer\\
	16 & dancer\\
	17 & footballer\\
	18 & industrialist\\
	19 & boxer\\
	20 & singer\\
	21 & violinist\\
\end{tabular}
\label{tbl:lookup}
\end{table}

\begin{figure}[!t]
\captionsetup[subfigure]{labelformat=empty}
\centering
	\subfloat[]
	{
			\includegraphics[width=0.5\columnwidth]{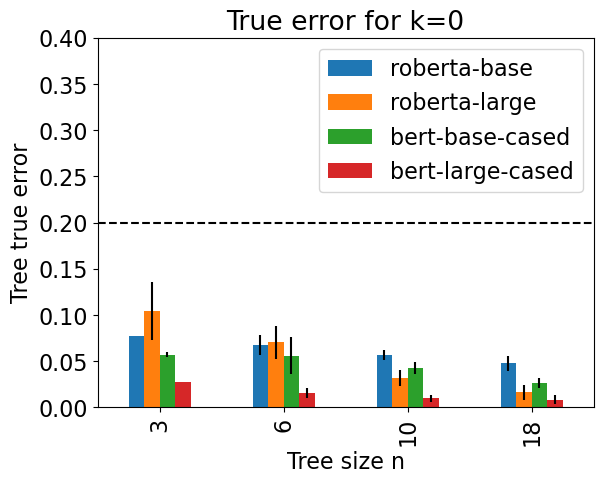}
			\includegraphics[width=0.5\columnwidth]{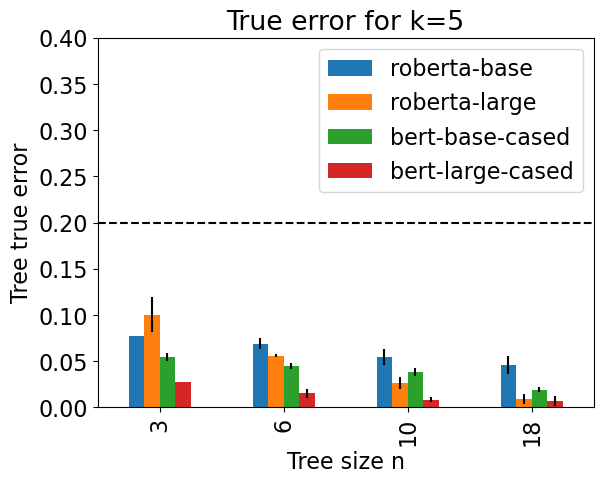}
		}
		\qquad
		\subfloat[]
		{
				\includegraphics[width=0.5\columnwidth]
				{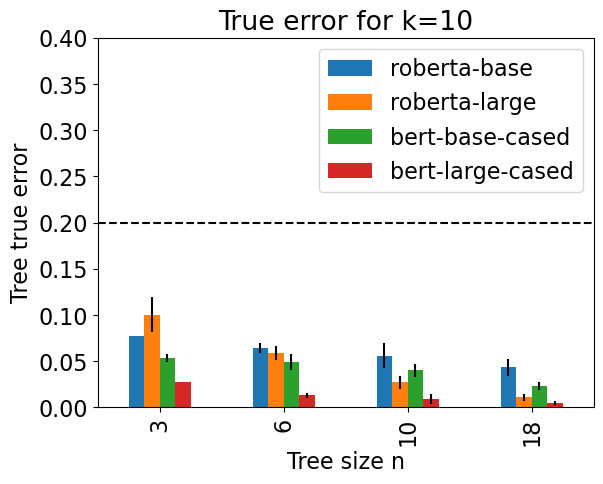}
				\includegraphics[width=0.5\columnwidth]
				{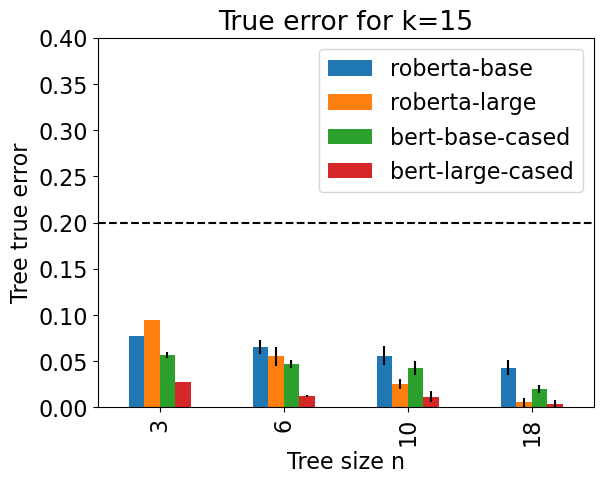}
			}
			\caption{\trepac true error for $k=[0,5,10,15]$ and $n=[3,6,10,18]$. Increasing the number of nodes decreases the true error. 
			}
			\label{fig:true_error}
		\end{figure}
		\begin{figure}[!t]
			\captionsetup[subfigure]{labelformat=empty}
			\centering
				\subfloat[]
				{
					\includegraphics[width=0.5\columnwidth]{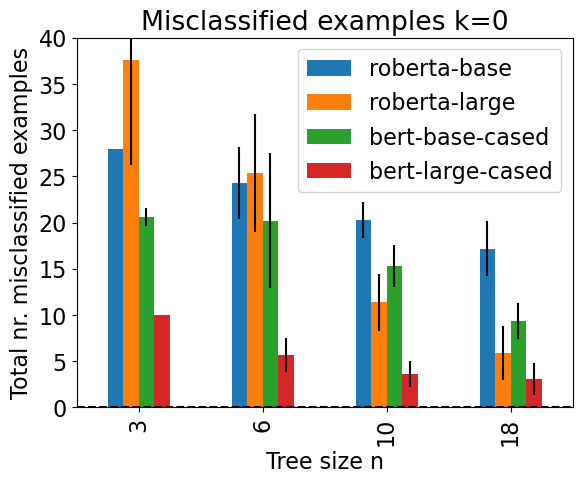}
					\includegraphics[width=0.5\columnwidth]{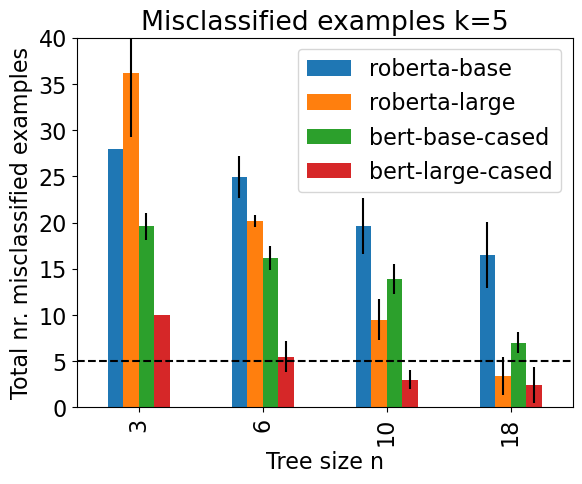}
				}
				\qquad
				\subfloat[]
				{
					\includegraphics[width=0.5\columnwidth]
					{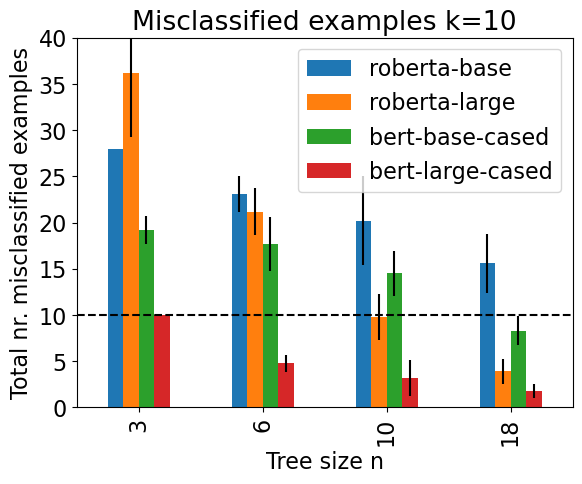}
					\includegraphics[width=0.5\columnwidth]
					{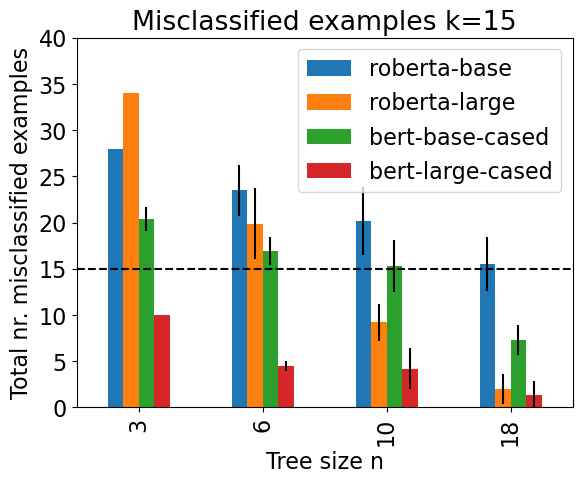}
				}
				\caption{\trepac total number of misclassified examples for $k=[0,5,10,15]$ and $n=[3,6,10,18]$.  Increasing the number of internal nodes of the tree reduces the total number of misclassified examples.
				}
				\label{fig:mis_examples_true}
			\end{figure}
			
			\begin{table}[!t]
				\begingroup
				\setlength{\tabcolsep}{4pt} 
				\caption{Mean training error, true error, misclassified training and true examples, and query/tree extraction times using \trepac for $\epsilon=0.2$, $\delta=0.1$, $n = 10$ and $k=10$ ($m=409$). The binary entropy is used as criterion for the selection of the best splits (std. dev. are in parenthesis). 
				}
				\small	
				\centering
				\begin{tabular}{llll}
					
					\hline
					Model & Measure & \trepac  \\
					\hline \hline
					\multirow{6}{*}{RoBERTa-base} 
					& Training error & 0.020 (0.011) \\
					& Training misclassified ex. & 8.2 (4.4) \\
					& True error & 0.056 (0.013) \\
					& True misclassified ex. & 20.2 (4.8) \\
					& Querying time & 651.52 (10.96) \\
					& Tree extraction time & 0.00070 (0.00002) \\
					
					\multirow{6}{*}{RoBERTa-large}
					& Training error & 0.013 (0.006) \\
					& Training misclassified ex. & 5.2 (2.6) \\
					& True error & 0.027 (0.007) \\
					& True misclassified ex. & 9.8 (2.5) \\
					& Querying time & 1288.38 (13.52) \\
					& Tree extraction time & 0.00072 (0.00002) \\
					
					\multirow{6}{*}{BERT-base-cased}
					& Training error & 0.022 (0.009) \\
					& Training misclassified ex. & 9.1 (3.7) \\
					& True error & 0.040 (0.007) \\
					& True misclassified ex. & 14.5 (2.5) \\
					& Querying time & 654.61 (7.97) \\
					& Tree extraction time & 0.00070 (0.00003) \\
					
					\multirow{6}{*}{BERT-large-cased}
					& Training error & 0.002 (0.003) \\
					& Training misclassified ex. & 1.0 (1.1) \\
					& True error & 0.009 (0.005) \\
					& True misclassified ex. & 3.2 (1.9) \\
					& Querying time & 1275.51 (16.70) \\
					& Tree extraction time & 0.00070 (0.00003) \\
					
					& $~$ & $~$ & $~$ \\
				\end{tabular}
				\label{tbl:results}
				\endgroup
			\end{table}
			
			\section{True Error and Misclassified Examples }\label{sec:addtables}
			The true error is usually impossible to calculate since it requires knowledge about the classification of all examples. However, in our experiments, the number of possible examples is finite, allowing us to calculate the true error and the total number of misclassified examples. This was done with the purpose of further confirming the results presented in the main text and as a sanity check.
			Figures~\ref{fig:true_error} and~\ref{fig:mis_examples_true} show results regarding the true error and the total number of misclassified examples for all models for $k=[0,5,10,15]$ and $n=[3,6,10,18]$, respectively. Table~\ref{tbl:results} shows all metrics for trees distilled using \trepac for $\epsilon=0.2$, $\delta=0.1$, $n = 6$ and $k=10$ ($m=409$).

			\section{Time for Probing LMs}\label{sec:timesprob}
			
			Table~5 reports the time taken for querying the language models.  We used the huggingface API.\footnote{\url{https://huggingface.co}}
			\begin{table*}[!t]
				\begingroup
				\small
				\setlength{\tabcolsep}{10pt} 
				\caption{Mean values of the time (in secs) taken for probing the BERT-based models (std. dev. are shown in parenthesis). 
				}
				
				\small
				\begin{tabular}{ccllll}
					\hline 
					$k$ & $m$ & RoBERTa-base & RoBERTa-large & BERT-base-cased  & BERT-large-cased \\
					\hline \hline
					
					\multirow{4}{*}{$k=0$}  
					& 124 & 194.06 (3.79) & 389.40 (11.08) & 199.80 (3.20) & 385.20 (9.02) \\
					& 201 & 323.52 (9.56) & 626.87 (10.52) & 320.24 (6.15) & 627.91 (9.58) \\
					& 257 & 408.87 (10.04) & 795.52 (8.66) & 408.42 (10.14) & 807.45 (17.38) \\
					& 321 & 517.63 (30.02) & 1002.42 (5.42) & 516.04 (7.44) & 1001.65 (15.75) \\
					\hline
					
					\multirow{4}{*}{$k=5$}  
					& 204 & 322.88 (6.30) & 633.72 (15.85) & 324.77 (9.16) & 642.10 (10.35) \\
					& 280 & 449.09 (10.68) & 871.86 (10.95) & 442.39 (9.03) & 877.34 (17.12) \\
					& 336 & 533.46 (7.38) & 1044.84 (14.01) & 538.86 (7.41) & 1048.94 (12.99) \\
					& 401 & 641.22 (8.07) & 1251.11 (18.04) & 640.29 (7.58) & 1260.73 (20.11) \\
					\hline
					
					\multirow{4}{*}{$k=10$}  
					& 277 & 446.62 (7.05) & 893.27 (69.46) & 437.94 (7.62) & 863.31 (15.34) \\
					& 353 & 561.80 (7.90) & 1106.79 (10.59) & 564.47 (10.66) & 1111.26 (15.22) \\
					& 409 & 651.52 (10.96) & 1288.38 (13.52) & 654.61 (7.97) & 1275.51 (16.70) \\
					& 474 & 750.88 (11.76) & 1483.87 (21.48) & 751.78 (7.20) & 1473.59 (21.65) \\
					\hline
					
					\multirow{4}{*}{$k=15$}  
					& 349 & 561.15 (10.98) & 1091.17 (13.01) & 552.31 (5.30) & 1087.28 (7.38) \\
					& 425 & 676.90 (10.59) & 1320.29 (10.90) & 677.72 (5.45) & 1338.88 (10.68) \\
					& 481 & 761.37 (9.46) & 1517.56 (18.20) & 768.32 (12.47) & 1505.58 (11.20) \\
					& 546 & 866.33 (9.18) & 1706.60 (10.81) & 873.26 (11.76) & 1707.74 (21.86) \\
					\hline
				\end{tabular}
				\endgroup
				
				\label{tbl:resultstime}
			\end{table*}
			
			\section{Decision Tree}\label{sec:dectreefull}
			Figure~\ref{fig:surrogate_trees} shows two decision trees extracted from RoBERTa-base. 
			\begin{figure*}[!t]
				\centering
					\subfloat[]
					{
						\includegraphics[width=0.6\textwidth]{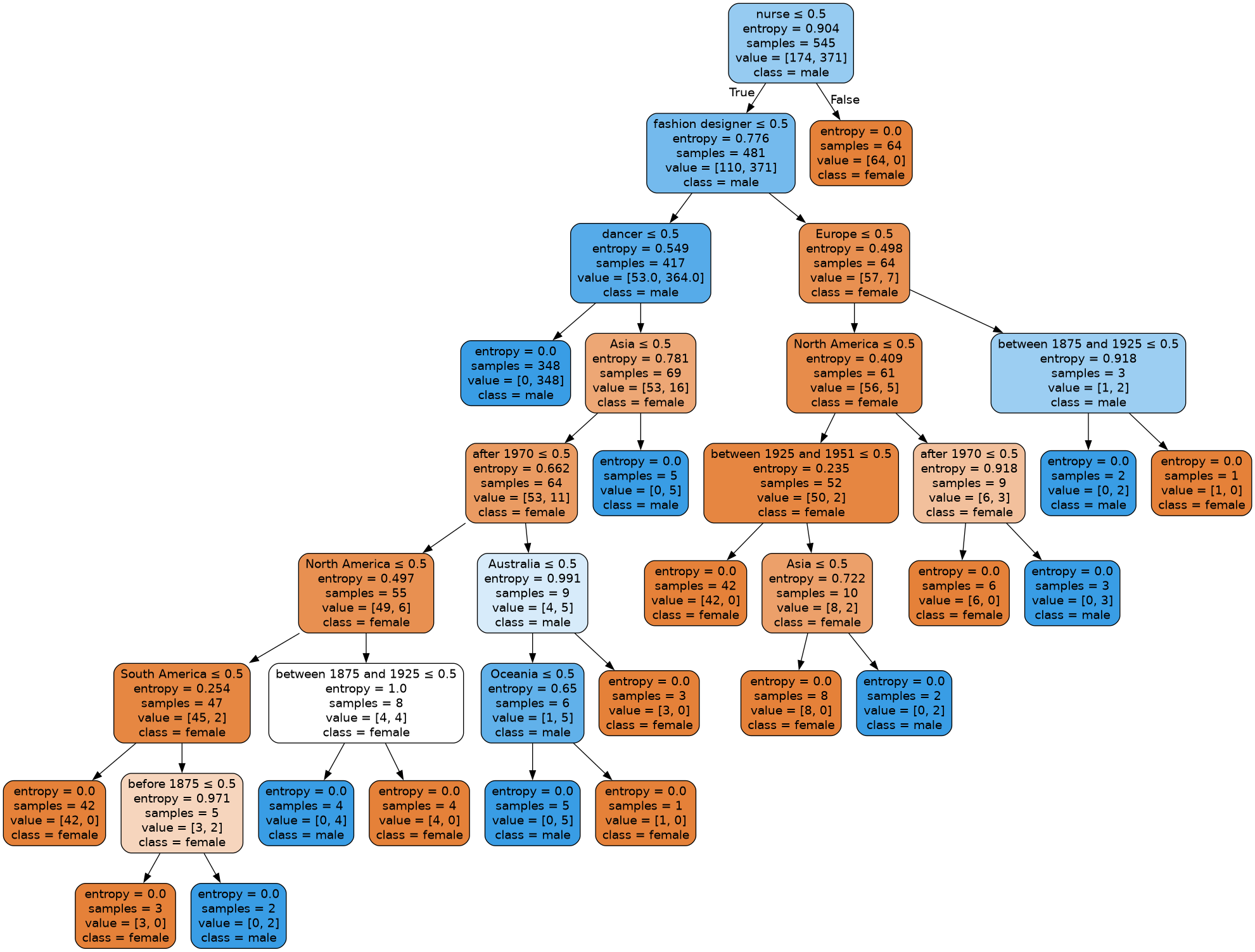}
						
					} 
					\qquad
					\subfloat[]
					{

						\includegraphics[width=0.3\textwidth]{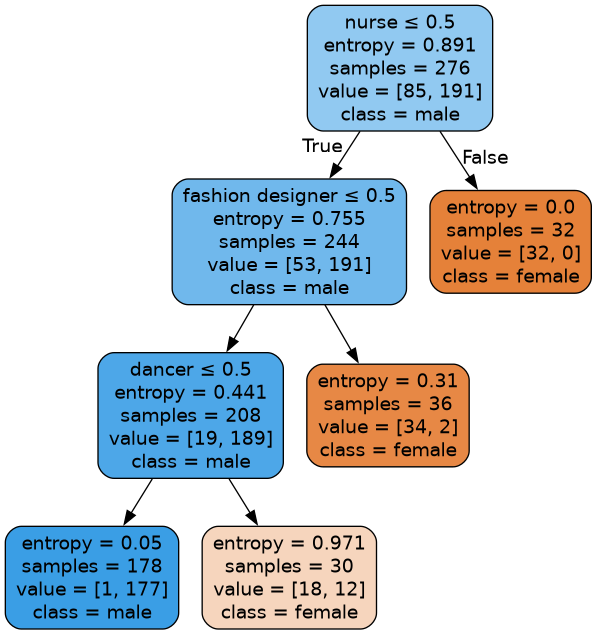}
					} 
					\caption{Decision trees for RoBERTa-base. The tree on the left is extracted with $n=18$ and $m=546$ examples ($k=15$). The tree on the right is extracted with $n=3$ and $m=277$ examples ($k=10$), see \cref{tab:sample}. It can be noticed that the most relevant features are   the  occupations, followed by  birth period and location. 
					}
					\label{fig:surrogate_trees}
				\end{figure*} 
				
			\end{document}